\documentclass{article}
\usepackage{spconf,amsmath,graphicx}
\usepackage{cite}
\usepackage{hyperref}       
\usepackage{url}            
\usepackage{booktabs}       
\usepackage{amsfonts}       
\usepackage{nicefrac}       
\usepackage{microtype}      
\usepackage{xcolor}         
\usepackage{adjustbox}
\usepackage{amsmath}
\usepackage{amsthm}
\usepackage{amsfonts,amssymb} 
\usepackage{graphicx}
\usepackage{algorithm}
\usepackage{algorithmic}
\usepackage{multirow}
\usepackage{color}

\usepackage{amsmath}
\DeclareMathOperator*{\argmax}{arg\,max}

\def\eg{{\emph{e.g., }}}
\def\ie{{\emph{i.e., }}}
\newcommand{\bihan}[1]{{\color{black}#1}}

\title{Benchmarking Adversarial Robustness of Image Shadow Removal with Shadow-adaptive Attacks}
\name{Chong Wang$^1$, Yi Yu$^{1,2}$, Lanqing Guo$^1$, and Bihan Wen$^{1*}$\thanks{$^*$Bihan Wen is the corresponding author.}}
\address{$^1$School of Electrical and Electronic Engineering, Nanyang Technological University, Singapore\\
$^2$ROSE Lab, Interdisciplinary Graduate Programme, Nanyang Technological University, Singapore}
\begin{document}
%

\newtheorem{theorem}{Theorem}

\maketitle
\begin{abstract}
\bihan{Shadow removal is a task aimed at erasing regional shadows present in images and reinstating visually pleasing natural scenes with consistent illumination.}
\bihan{While recent deep learning techniques have demonstrated impressive performance in image shadow removal, their robustness against adversarial attacks remains largely unexplored.}
\bihan{Furthermore, many existing attack frameworks typically allocate a uniform budget for perturbations across the entire input image, which may not be suitable for attacking shadow images.}
\bihan{This is primarily due to the unique characteristic of spatially varying illumination within shadow images.}
\bihan{
In this paper, we propose a novel approach, called shadow-adaptive adversarial attack. Different from standard adversarial attacks, our attack budget is adjusted based on the pixel intensity in different regions of shadow images. 
Consequently, the optimized adversarial noise in the shadowed regions becomes visually less perceptible while permitting a greater tolerance for perturbations in non-shadow regions.} 
\bihan{The proposed shadow-adaptive attacks naturally align with the varying illumination distribution in shadow images, resulting in perturbations that are less conspicuous.}
\bihan{Building on this, we conduct a comprehensive empirical evaluation of existing shadow removal methods, subjecting them to various levels of attack on publicly available datasets.}

\end{abstract}
\begin{figure}[!t]
\centering
\includegraphics[width=1.\linewidth]{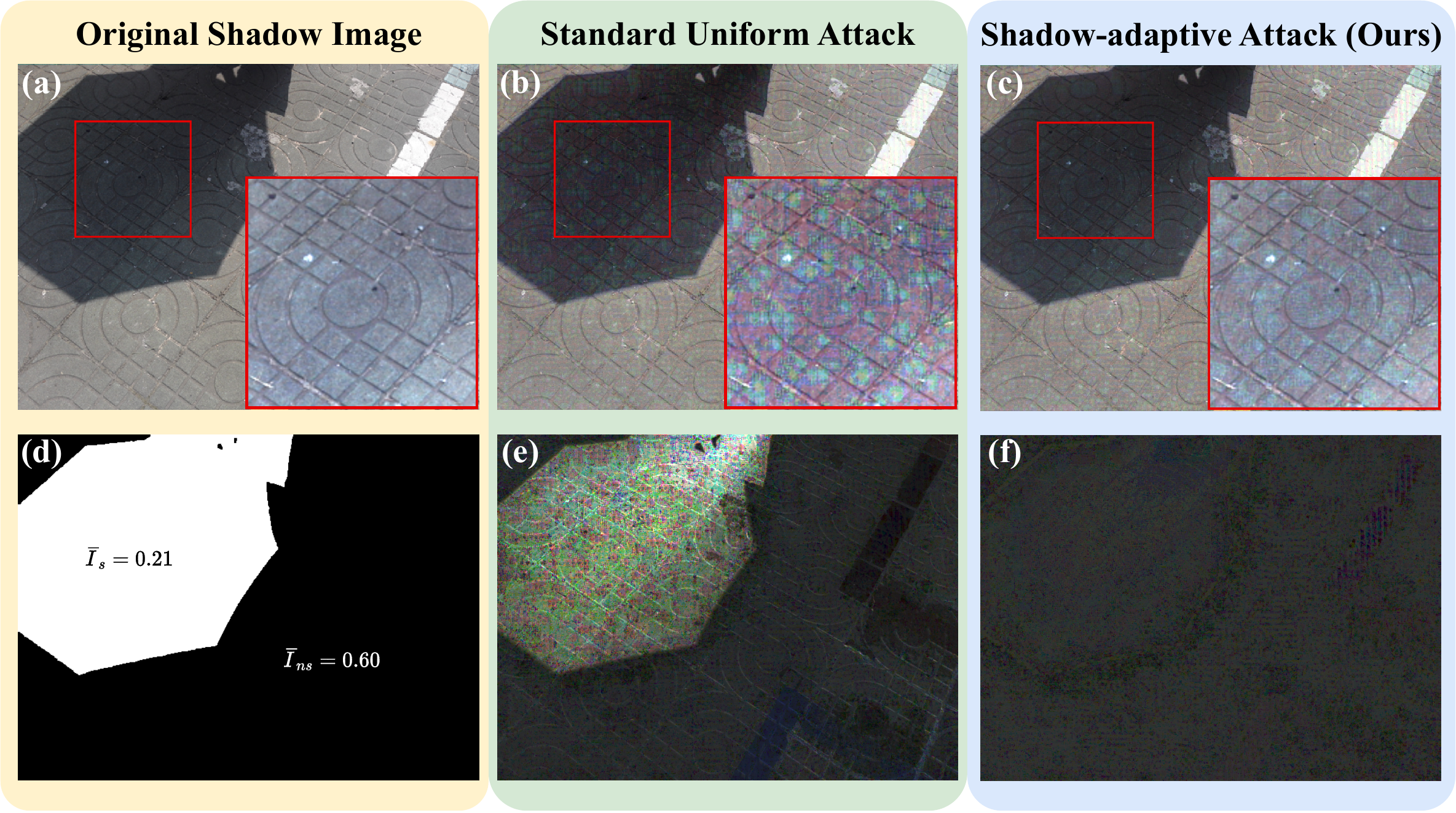} 
\vspace{-8mm}
\caption{
\bihan{Comparison of the adversarial examples from (a) a shadow image from the ISTD+ dataset, using (b) the standard uniform attack, and (c) the proposed shadow-adaptive attack.}
\bihan{(d) The shadow mask indicates that the average pixel intensity within the shadow and non-shadow regions, denoted as $\bar{I}_s$ and $\bar{I}_{ns}$, are significantly different. (e) and (f) show the normalized perturbation $\delta / I$ generated by the two attacks, which better reflect the visual sensitivity of the adversarial perturbations.}
The zoomed-in subfigures in (a)-(c) and the normalized perturbations in (e)$\&$(f) are multiplied by $3\times$ for better visibility.
}
\vspace{-0.2cm}
\label{fig_attack_comparison} 
\end{figure}

\section{Introduction}
\label{sec:intro}
Shadows are a ubiquitous occurrence in natural images, particularly when light sources are partially or completely obstructed. 
The presence of shadow can significantly impair various subsequent computer vision tasks, \eg object detection~\cite{sanin2010improved}, tracking~\cite{nadimi2004physical}, and face recognition~\cite{zhang2018improving}.
Consequently, shadow removal has garnered substantial attention from researchers and is recognized as one of the fundamental image restoration tasks in recent years~\cite{zhang2018improving, yang2012shadow}.
Traditional approaches to shadow removal have primarily relied on illumination modeling~\cite{xiao2013fast} or physics priors in the shadow images, \eg region consistency~\cite{guo2012paired} and image gradients~\cite{gryka2015learning}.
Unfortunately, such methods using hand-crafted priors are often excessively idealistic, thus limiting the practical performance when dealing with complex real-world shadow images.
Recent advancements in shadow removal have been driven by learning-based methods harnessing the potential of extensive datasets, showing superior performance~\cite{wan2022style, guo2023shadowformer, guo2023shadowdiffusion, zhu2022bijective}.
In these cases, networks are trained in an end-to-end fashion, enabling them to acquire an inherent mapping between shadow and their corresponding shadow-free images.

However, these learning-based methods, due to their ability to adapt to the training data, often display a heightened sensitivity to specific input images. 
Consequently, trained networks are susceptible to vulnerabilities, and their predictions can significantly deteriorate in response to subtle, visually imperceptible perturbations in the input images~\cite{madry2017towards,yu2023backdoor}.
The robustness against such perturbations, namely adversarial robustness, has become a popular topic that received growing attention in computational imaging as well as other vision communities~\cite{apostolidis2021survey, yu2022towards}.
Despite its importance, there has been a notable absence of comprehensive evaluations regarding the adversarial robustness of deep shadow removal models, despite its status as an ongoing research challenge.

Current attack frameworks typically generate perturbations with a spatially uniform budget allocated to each pixel in the input images.
However, shadow images inherently exhibit regional illumination inconsistencies, resulting in notable disparities in pixel intensity between shadowed and non-shadowed regions. Consequently, existing attacks employing uniform budgets may struggle to generate visually imperceptible adversarial noise, particularly within the shadowed regions, as shown in Figure~\ref{fig_attack_comparison}(b).
In this work, 
we embark on a pioneering exploration of the robustness of learning-based shadow removal models against adversarial attacks. 
We propose a novel shadow-adaptive adversarial attack, meticulously designed for shadow removal tasks. 
In this attack, the budget for perturbing each pixel in input images is dynamically adjusted according to the pixel intensity in different regions.
Unlike existing uniform attacks, our proposed shadow-adaptive approach naturally aligns with the spatially varying illumination present in shadow images, resulting in perturbations that are notably less perceptible. 
Research~\cite{wieberconstant} has shown that human vision is more sensitive to disturbances when normalized by the original intensity, while the proposed adaptive attacks could still keep the stealth with respect to the normalized perturbation as shown in Figure~\ref{fig_attack_comparison} (e) and (f).
Building upon this, we conduct a comprehensive benchmark evaluation of existing deep shadow removal networks with various attack levels on two public datasets, namely ISTD and ISTD+.
To our knowledge, no work to date has thoroughly studied the adversarial robustness of deep shadow removal models, and we believe that our benchmark and empirical analysis could provide insights for building up robust shadow removal methods.



\begin{figure*}[t]
\centering
\includegraphics[width=1.0\linewidth]{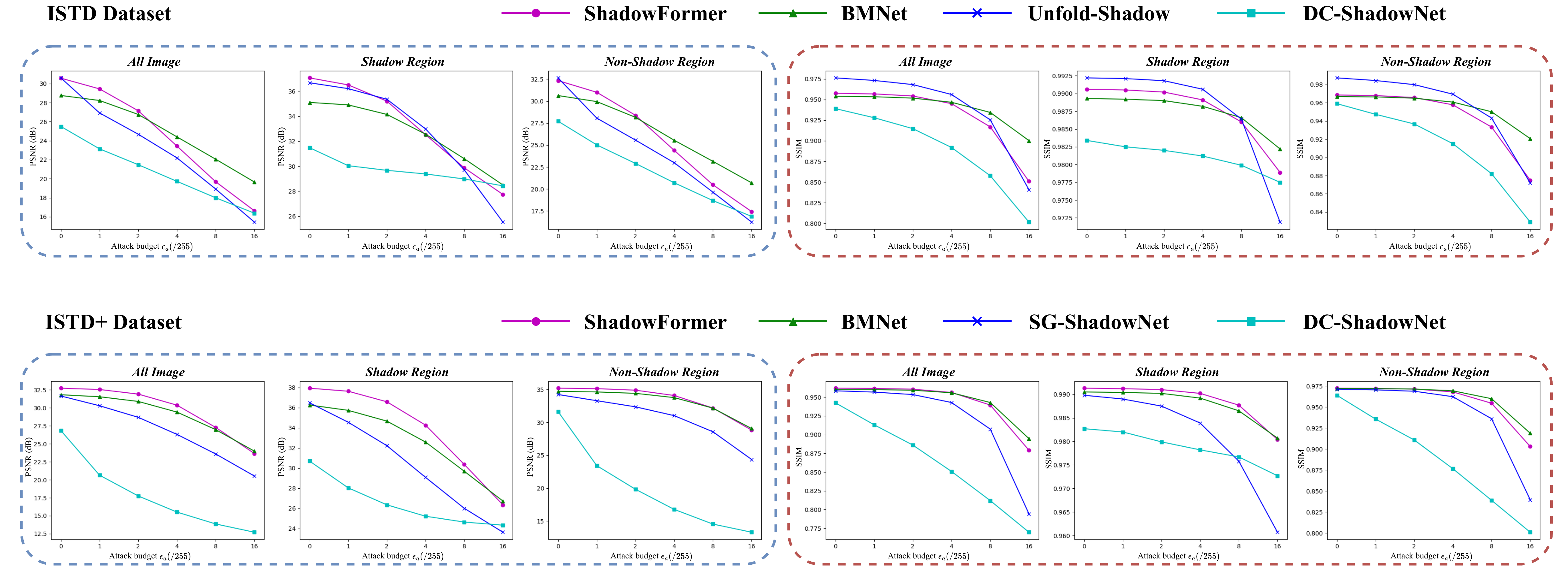} 
\vspace{-5mm}
\caption{
Adversarial robustness of five existing shadow removal models against our proposed shadow-adaptive attack with various attack budget $\epsilon_a$ evaluated by PSNR and SSIM on ISTD~\cite{wang2018stacked} and ISTD+~\cite{le2019shadow} datasets.
Each subfigure in the dashed line indicates the results of the whole image, shadow region, and non-shadow region, respectively.
}
\vspace{-1mm}
\label{fig_benchmark_curve} 
\end{figure*}

\section{Adversarial Attacks for Shadow Removal Tasks}
The objective of adversarial attacks is to introduce perturbations into input images that are visually imperceptible, yet result in a substantial degradation of the output quality produced by a deep shadow removal model.
Given the absence of prior research addressing adversarial attacks in the context of shadow removal, we initiate our study by formally defining this problem as follows.

Given an input shadow image $I$, a shadow-free estimation $\hat{I}$ can be predicted by a well-trained shadow removal network $f_\theta(\cdot)$, that is $\hat{I} = f_\theta(I)$.
To deteriorate the shadow removal results, a small unnoticeable perturbation $\delta$ is optimized to contaminate each input shadow image, as
\begin{equation}\label{eq_ori_delta}
    \delta = \argmax_{\lVert\delta\rVert_\infty \leq \epsilon} D\langle f_\theta(I), f_\theta(I+\delta)\rangle,
\end{equation}
where $\epsilon$ is the attack budget to constrain the maximum value of $\delta$ and $D\langle\cdot, \cdot\rangle$ denotes the metrics that measure the distance of the attacked shadow removal result and its corresponding original one.
In this work, we mainly focus on the $\ell_2$ distance (\ie $D\langle\cdot, \cdot\rangle=\lVert \cdot \rVert_2$), which is widely used in existing robustness studies~\cite{yu2022towards, song2023robust}.
The optimization problem in~\eqref{eq_ori_delta} can be solved iteratively via PGD by utilizing the first-order information about the network:
\begin{align}
    \tilde{\delta}^{t+1} &= \delta^t + \alpha \cdot \mathrm{sgn}(\nabla_{\delta} \lVert f_\theta(I) - f_\theta(I+\delta)\rVert_2),\\
    \delta^{t+1} &= \mathrm{clip}_{[-\epsilon, \epsilon]\cap[-I, 1-I]}(\tilde{\delta}^{t+1}), \label{eq_clip}
\end{align}
where $\mathrm{sgn}(\cdot)$ and $\nabla$ denote the sign function and gradient operation, respectively, $\alpha$ controls the step size during each iteration thus preventing noticeable modification on the attacked input images.
The step in~\eqref{eq_clip} prevents the generated perturbation from exceeding the attack budget $\epsilon$ while also maintaining the pixel intensity of attacked images lies in $[0, 1]$. 
The initial $\delta^0$ is sampled from the uniform distribution $U(-\epsilon, \epsilon)$, and the final perturbation $\delta$ is optimized for total $T$ iterations.

\section{Shadow-Adaptive Adversarial Attack}
The standard adversarial attack framework (\ref{eq_ori_delta}) generates perturbation with a uniform budget $\epsilon$ allocated to each pixel in input images. 
However, this uniform-budget approach may not ensure imperceptibility, particularly in the context of shadow removal. 
Shadow images, due to the occlusion of the light sources, exhibit a distinctive property characterized by variations in illumination between shadowed and non-shadowed regions. 
Specifically, the pixel intensity in shadow regions is notably lower than that in non-shadowed regions.
Consequently, perturbations in the shadow region can become conspicuous when generated using a uniform attack budget across the entire image.

To preserve visual imperceptibility, we propose a shadow-adaptive attack where the perturbation is under an adaptive budget according to each pixel intensity, that is 
\begin{equation}\label{eq_adap_delta}
    \delta = \argmax_{\lVert\delta / I \rVert_\infty \leq \epsilon} \lVert f_\theta(I), f_\theta(I+\delta)\rVert_2.
\end{equation}
Rather than directly controlling the attack noise $\delta$, we introduce a constraint related to the proportion of the perturbation relative to pixel intensity, denoted as the \textit{normalized perturbation} $\delta / I$.
This adaptive attack strategy yields perturbations that align with the illumination distribution within the input shadow image. 
In particular, the attack noise remains moderate in regions with low pixel intensity, i.e., shadow regions, to ensure imperceptibility, while permitting stronger perturbation in non-shadow regions, facilitating effective attacks.

To impose a fair comparison between the proposed adaptive adversarial attack and the uniform attacks, Theorem~\ref{theorem_1} shows that the maximal achievable perturbations of the two approaches are equivalent using the $\ell_1$-norm.
\begin{theorem}\label{theorem_1}
    Let $\epsilon_{a}$ denotes the attack budget in our proposed shadow-adaptive attacks in~\eqref{eq_adap_delta}, we can set the budget $\epsilon_{u} = \epsilon_a  \bar{I}  $ for standard uniform attack in~\eqref{eq_ori_delta} to achieve an equivalent attack strength, where $\bar{I}$ is the mean value of the input shadow image $I$.
\end{theorem}
\begin{proof}
Let $i$ denotes the pixel index in shadow input $I\in \mathbb{R}^n$, $\delta_a$ and $\delta_u$ denote the perturbation generated from the proposed adaptive attack~\eqref{eq_adap_delta} and uniform attacks~\eqref{eq_adap_delta}, respectively.
In the adaptive attack, since the perturbation on each pixel is bounded by its intensity as $\lvert \delta^{(i)}_a \rvert \leq \epsilon_a  I^{(i)}$, then we have
\begin{align*}
    \frac{1}{n} \rVert \delta_a \rVert_1 \leq \frac{1}{n} \sum_i^n \epsilon_a  I^{(i)}  = \epsilon_a  \frac{1}{n} \sum_i^n   I^{(i)}  = \epsilon_a \bar{I}.
\end{align*}
Similarly, in the uniform attacks, the perturbation is upper bounded by $\frac{1}{n}\rVert \delta_u \rVert_1 = \frac{1}{n} 
 \sum_i^n\lvert \delta^{(i)}_u \rvert \leq \epsilon_u= \epsilon_a  \bar{I}$.
Thus the maximum achievable perturbation of two types of attack is equivalent under the $\ell_1$-norm.
\end{proof}
In the following sections, we conduct an empirical evaluation of the proposed shadow-adaptive attack in comparison with the traditional uniform attack, as well as its attacking effectiveness on various existing deep shadow removal models.

\begin{figure*}[t]
\centering
\includegraphics[width=0.99\linewidth]{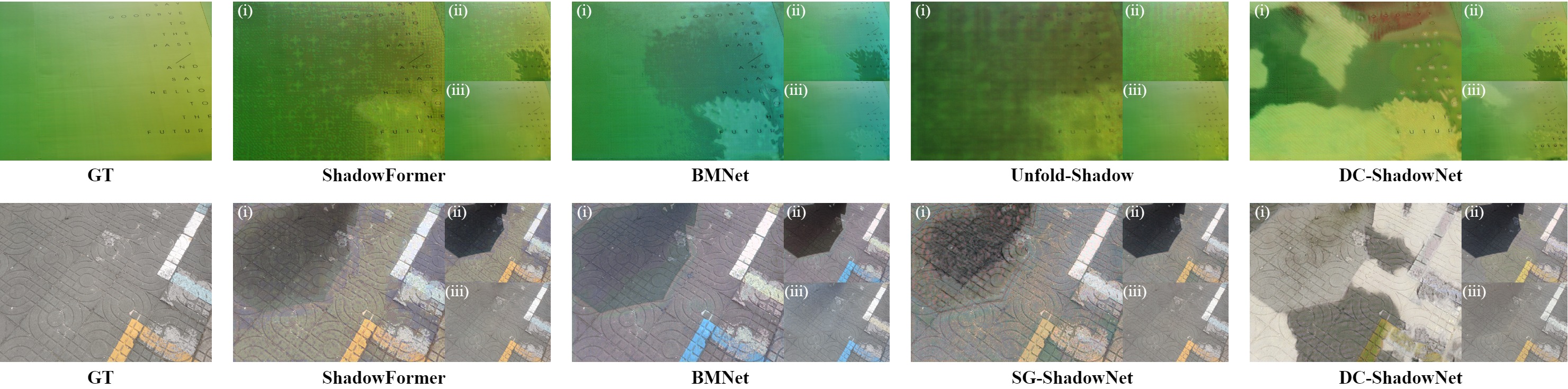} 
\vspace{-3mm}
\caption{
Visual comparison under our shadow-adaptive attack with budget $\epsilon_a$ = 16/255 on the ISTD dataset~\cite{wang2018stacked} (top row) and ISTD+ dataset~\cite{le2019shadow} (bottom row). 
Subfigures (i), (ii), and (iii) represent attacked shadow removal results, perturbed input shadow images, and original shadow removal results, respectively.
}
\vspace{-1mm}
\label{fig_benchmark_visual} 
\end{figure*}

\begin{table}[t]
\vspace{-2mm}
\centering
\vspace{-1mm}
\caption{The quantitative comparison of our proposed shadow-adaptive attack and standard uniform attacks.}
\label{tab_attack_comparison}
\adjustbox{width=0.9\linewidth}{
\begin{tabular}{cccccc}
\toprule
Dataset       &                  & \multicolumn{2}{c}{ISTD} & \multicolumn{2}{c}{ISTD+} \\ \midrule
Attack budget & Attack method & PSNR       & SSIM        & PSNR        & SSIM        \\ \midrule
\multirow{2}{*}{$\epsilon$ = 8/255} & Uniform & 19.40 & 0.9156 & 26.44 & 0.9404 \\
              & Ours adaptive    & 19.68      & 0.9168      & 27.28       & 0.9395      \\ \midrule
\multirow{2}{*}{$\epsilon$ = 16/255} & Uniform & 16.44 & 0.8526 & 22.74 & 0.8846 \\
              & Ours adaptive    & 16.62      & 0.8509      & 23.65       & 0.8791      \\ \midrule
             \multicolumn{2}{c}{Clean}     & 30.58     & 0.9574      & 32.73       & 0.9619\\
              
\bottomrule
\end{tabular}
\vspace{-1mm}
}
\end{table}

\section{Experiments}
\label{sec:benchmarking}
We first present an empirical comparison of the proposed shadow-adaptive attack and standard uniform attacks, to validate its effectiveness and stealth.
Based on that, we further conduct a comprehensive robustness evaluation of various existing learning-based shadow removal methods against the proposed shadow-adaptive attack.

\subsection{Implementation Details}
\noindent\textbf{Datasets and Metrics.}
We conduct evaluations on two benchmark datasets of shadow removal: (1) ISTD dataset~\cite{wang2018stacked} that includes 1,330 training and 540 testing triplets of shadow, shadow mask, and shadow-free images. (2) Adjusted ISTD (ISTD+) dataset~\cite{le2019shadow} that reduces the illumination inconsistency between the shadow and shadow-free images of ISTD. 
We adopt the Peak Signal-to-Noise Ratio (PSNR) and Structural Similarity (SSIM) as the evaluation metrics.
Besides, following the previous works~\cite{guo2023shadowdiffusion, wan2022style, zhu2022bijective}, we also compute these metrics of shadow and non-shadow regions separately. 

\noindent\textbf{Shadow Removal Models.}
We consider five state-of-the-art learning-based methods in shadow removal, including four supervised methods, namely ShadowFormer~\cite{guo2023shadowformer}, BMNet~\cite{zhu2022bijective}, Unfold-Shadow~\cite{zhu2022efficient} and SG-ShadowNet~\cite{wan2022style}, as well as one unsupervised method DC-ShadowNet~\cite{jin2021dc}.

\subsection{Shadow-adaptive Attack vs. Uniform Attack}\label{subsec_attack_compare}
\noindent\textbf{Attack Details.}
The adversarial attacks are implemented based on PGD~\cite{madry2017towards} for total iteration $T=20$.
For the proposed shadow-adaptive attack, we choose two different budgets $\epsilon_a = 8/255 $ and $16/255$.
To achieve an equivalent attack strength, the budget for the standard uniform attack is accordingly set to $\epsilon_{u} = \epsilon_a  \bar{I}$, where $\bar{I}$ is the average pixel intensity of each input shadow image.
The step size is adjusted according to the attack budget as $\alpha = \epsilon / 4$ for each experiment.

\noindent\textbf{Results and Analysis.}
A well-trained ShadowFormer~\cite{guo2023shadowformer} is adopted as the attack objective.
Table~\ref{tab_attack_comparison} shows the quantitative results of shadow removal after the proposed shadow-adaptive attack and standard uniform attack on ISTD and ISTD+ datasets, respectively, where the lower PSNR and SSIM indicate more aggressive attacks.
While our proposed shadow-adaptive attack could achieve comparable attack performance to the standard uniform one, more importantly, the generated perturbation is significantly more imperceptible, as shown in Figure~\ref{fig_attack_comparison} (b)$\&$(c).
Even when attack strength is large (\eg $\epsilon_a=16/255$), our adaptive adversarial noise still maintains its imperceptibility.
Compared to the absolute value of noise, studies~\cite{wieberconstant} show that human vision is more sensitive to disturbances when scaled by the original intensity.
Thus we compared the normalized perturbation $\delta/I$ to reflect their visual sensitivity, as presented in Figure~\ref{fig_attack_comparison} (e)$\&$(f).
With the merit of the shadow-adaptive attack budget, the normalized perturbation from our attack keeps consistent stealth across the whole image while that from the uniform attack is significantly noticeable in the shadow region.

\subsection{Benchmarking the Adversarial Robustness}
\noindent\textbf{Attack Details.}
We adopt the proposed shadow-adaptive attack to evaluate the robustness of existing deep shadow removal models.
We set the attack budgets $\epsilon_a \in \{1/255, 2/255, \\
4/255, 8/255, 16/255\}$, regarding to $\ell_\infty$-norm.
The settings for step size $\alpha$ and iteration $T$ are the same as Section~\ref{subsec_attack_compare}.

\noindent\textbf{Results and Remarks.}
Figure~\ref{fig_benchmark_curve} compares the PSNR and SSIM results of existing deep shadow removal models against our proposed shadow-adaptive attacks under various attack budgets $\epsilon_a$.
On both datasets, our shadow-adaptive attack can effectively degrade the performance of existing methods.
The adversarial attack is more effective on the ISTD dataset, where all the baselines dropped around 10dB on PSNR when the attack strength is large ($\epsilon_a = 16/255$).
The robustness against adversarial attacks is also different between shadow and non-shadow regions.
For those supervised models, they are more robust in the shadow region with respect to PSNR on the ISTD dataset and vice versa on the ISTD+ dataset.
However, for the unsupervised model DC-ShadowNet~\cite{jin2021dc}, the robustness is stronger in the shadow region compared to the non-shadow region on both datasets.
The visual results in Figure~\ref{fig_benchmark_visual} further justified the observation above.
For those supervised methods, there are evident artifacts in the non-shadow region on the ISTD dataset (top row in Figure~\ref{fig_benchmark_visual}) while much more distortion in the shadow region on the ISTD+ dataset (bottom row in Figure~\ref{fig_benchmark_visual}).
As for the DC-ShadowNet~\cite{jin2021dc}, there exist substantial block-wise artifacts in the non-shadow region while the corruption in the shadow region is relatively milder.

\noindent\textbf{Discussion.}
From the visual example in Figure~\ref{fig_benchmark_visual}, although our adaptive perturbation is almost invisible in subfigures (ii), the prediction results of each method in subfigures (i) are significantly degraded compared to the original shadow removal results in subfigures (iii).
The different robustness between shadow and non-shadow regions provides us insight into designing a robust model that possesses balanced robustness in these two regions, which could be crucial in shadow removal.
We will be devoted to this robust solution in future work.


\section{Conclusion}
\label{sec:conclusion}
In this work, we introduce a shadow-adaptive adversarial attack, which is specifically designed for the shadow removal task. 
We conduct both theoretical and empirical comparisons of the proposed adaptive attack with standard uniform attacks, and notably, our adaptive attack could achieve better imperceptibility, especially in the shadow region.
Based on that, we further present a comprehensive empirical evaluation on the robustness of existing deep shadow removal models against the proposed shadow-adaptive attack.
We believe our benchmark could provide insight into designing robust shadow removal models against adversarial attacks.

\bibliographystyle{IEEEbib}
\bibliography{refs}

\begin{thebibliography}{10}

\bibitem{sanin2010improved}
Andres Sanin, Conrad Sanderson, and Brian~C Lovell,
\newblock ``Improved shadow removal for robust person tracking in surveillance
  scenarios,''
\newblock in {\em 2010 20th International Conference on Pattern Recognition}.
  IEEE, 2010, pp. 141--144.

\bibitem{nadimi2004physical}
Sohail Nadimi and Bir Bhanu,
\newblock ``Physical models for moving shadow and object detection in video,''
\newblock {\em IEEE transactions on pattern analysis and machine intelligence},
  vol. 26, no. 8, pp. 1079--1087, 2004.

\bibitem{zhang2018improving}
Wuming Zhang, Xi~Zhao, Jean-Marie Morvan, and Liming Chen,
\newblock ``Improving shadow suppression for illumination robust face
  recognition,''
\newblock {\em IEEE transactions on pattern analysis and machine intelligence},
  vol. 41, no. 3, pp. 611--624, 2018.

\bibitem{yang2012shadow}
Qingxiong Yang, Kar-Han Tan, and Narendra Ahuja,
\newblock ``Shadow removal using bilateral filtering,''
\newblock {\em IEEE Transactions on Image processing}, vol. 21, no. 10, pp.
  4361--4368, 2012.

\bibitem{xiao2013fast}
Chunxia Xiao, Ruiyun She, Donglin Xiao, and Kwan-Liu Ma,
\newblock ``Fast shadow removal using adaptive multi-scale illumination
  transfer,''
\newblock in {\em Computer Graphics Forum}. Wiley Online Library, 2013,
  vol.~32, pp. 207--218.

\bibitem{guo2012paired}
Ruiqi Guo, Qieyun Dai, and Derek Hoiem,
\newblock ``Paired regions for shadow detection and removal,''
\newblock {\em IEEE transactions on pattern analysis and machine intelligence},
  vol. 35, no. 12, pp. 2956--2967, 2012.

\bibitem{gryka2015learning}
Maciej Gryka, Michael Terry, and Gabriel~J Brostow,
\newblock ``Learning to remove soft shadows,''
\newblock {\em ACM Transactions on Graphics (TOG)}, vol. 34, no. 5, pp. 1--15,
  2015.

\bibitem{wan2022style}
Jin Wan, Hui Yin, Zhenyao Wu, Xinyi Wu, Yanting Liu, and Song Wang,
\newblock ``Style-guided shadow removal,''
\newblock in {\em European Conference on Computer Vision}. Springer, 2022, pp.
  361--378.

\bibitem{guo2023shadowformer}
Lanqing Guo, Siyu Huang, Ding Liu, Hao Cheng, and Bihan Wen,
\newblock ``Shadowformer: Global context helps image shadow removal,''
\newblock {\em arXiv preprint arXiv:2302.01650}, 2023.

\bibitem{guo2023shadowdiffusion}
Lanqing Guo, Chong Wang, Wenhan Yang, Siyu Huang, Yufei Wang, Hanspeter
  Pfister, and Bihan Wen,
\newblock ``Shadowdiffusion: When degradation prior meets diffusion model for
  shadow removal,''
\newblock in {\em Proceedings of the IEEE/CVF Conference on Computer Vision and
  Pattern Recognition}, 2023, pp. 14049--14058.

\bibitem{zhu2022bijective}
Yurui Zhu, Jie Huang, Xueyang Fu, Feng Zhao, Qibin Sun, and Zheng-Jun Zha,
\newblock ``Bijective mapping network for shadow removal,''
\newblock in {\em Proceedings of the IEEE/CVF Conference on Computer Vision and
  Pattern Recognition}, 2022, pp. 5627--5636.

\bibitem{madry2017towards}
Aleksander Madry, Aleksandar Makelov, Ludwig Schmidt, Dimitris Tsipras, and
  Adrian Vladu,
\newblock ``Towards deep learning models resistant to adversarial attacks,''
\newblock {\em arXiv preprint arXiv:1706.06083}, 2017.

\bibitem{yu2023backdoor}
Yi~Yu, Yufei Wang, Wenhan Yang, Shijian Lu, Yap-Peng Tan, and Alex~C Kot,
\newblock ``Backdoor attacks against deep image compression via adaptive
  frequency trigger,''
\newblock in {\em Proceedings of the IEEE/CVF Conference on Computer Vision and
  Pattern Recognition}, 2023, pp. 12250--12259.

\bibitem{apostolidis2021survey}
Kyriakos~D Apostolidis and George~A Papakostas,
\newblock ``A survey on adversarial deep learning robustness in medical image
  analysis,''
\newblock {\em Electronics}, vol. 10, no. 17, pp. 2132, 2021.

\bibitem{yu2022towards}
Yi~Yu, Wenhan Yang, Yap-Peng Tan, and Alex~C Kot,
\newblock ``Towards robust rain removal against adversarial attacks: A
  comprehensive benchmark analysis and beyond,''
\newblock in {\em Proceedings of the IEEE/CVF Conference on Computer Vision and
  Pattern Recognition}, 2022, pp. 6013--6022.

\bibitem{wieberconstant}
Fechner Gustav~Theodor, Howes D~H, and Boring E~G,
\newblock ``Elements of psychophysics,''
\newblock 1966.

\bibitem{wang2018stacked}
Jifeng Wang, Xiang Li, and Jian Yang,
\newblock ``Stacked conditional generative adversarial networks for jointly
  learning shadow detection and shadow removal,''
\newblock in {\em Proceedings of the IEEE conference on computer vision and
  pattern recognition}, 2018, pp. 1788--1797.

\bibitem{le2019shadow}
Hieu Le and Dimitris Samaras,
\newblock ``Shadow removal via shadow image decomposition,''
\newblock in {\em Proceedings of the IEEE/CVF International Conference on
  Computer Vision}, 2019, pp. 8578--8587.

\bibitem{song2023robust}
Zhenbo Song, Zhenyuan Zhang, Kaihao Zhang, Wenhan Luo, Zhaoxin Fan, Wenqi Ren,
  and Jianfeng Lu,
\newblock ``Robust single image reflection removal against adversarial
  attacks,''
\newblock in {\em Proceedings of the IEEE/CVF Conference on Computer Vision and
  Pattern Recognition}, 2023, pp. 24688--24698.

\bibitem{zhu2022efficient}
Yurui Zhu, Zeyu Xiao, Yanchi Fang, Xueyang Fu, Zhiwei Xiong, and Zheng-Jun Zha,
\newblock ``Efficient model-driven network for shadow removal,''
\newblock in {\em Proceedings of the AAAI conference on artificial
  intelligence}, 2022, vol.~36, pp. 3635--3643.

\bibitem{jin2021dc}
Yeying Jin, Aashish Sharma, and Robby~T Tan,
\newblock ``Dc-shadownet: Single-image hard and soft shadow removal using
  unsupervised domain-classifier guided network,''
\newblock in {\em Proceedings of the IEEE/CVF International Conference on
  Computer Vision}, 2021, pp. 5027--5036.

\end{thebibliography}

\end{document}